\documentclass[conference]{IEEEtran} 
\IEEEoverridecommandlockouts

\pdfoutput=1
\usepackage[T1]{fontenc}
\usepackage[english]{babel}
\usepackage[pdftex]{graphicx}   % for pdf, bitmapped graphics files
\usepackage{epsfig}             % for postscript graphics files
\usepackage{amsmath}            % assumes amsmath package installed
\usepackage{amssymb}            % assumes amsmath package installed
\usepackage{mathtools}
\usepackage{multicol}
\usepackage[bookmarks=true]{hyperref}
\usepackage{xcolor}
\usepackage{lipsum}
\usepackage{scrextend}
\usepackage{algorithm}
\usepackage{algpseudocode}
\usepackage{multirow}
\usepackage{cite}

\usepackage[inkscapelatex=false]{svg}

\usepackage{bm}

\usepackage{MnSymbol}
\usepackage{mathdots}
\usepackage{amsthm}
\usepackage{empheq}

\theoremstyle{plain}
\newtheorem{theorem}{Theorem}[section]

\theoremstyle{definition}

\theoremstyle{remark}

\newcommand{\R}{\mathbb{R}}
\newcommand{\norm}[1]{\lVert #1 \rVert}

\providecommand{\tr}[1]{\mathrm{tr}(#1)}

\newcommand{\Vector}[3]{\prescript{#1}{}{\bm{#2}}_{#3}}

\newcommand{\hatVector}[3]{\prescript{#1}{}{\hat{\bm{#2}}}_{#3}}

\newcommand{\barVector}[3]{\prescript{#1}{}{\bar{\bm{#2}}}_{#3}}

\newcommand{\Rn}[1]{\R^{#1}}

\newcommand{\figref}[1]{Fig.~\ref{fig:#1}}
\newcommand{\secref}[1]{Sec.~\ref{sec:#1}}

\newcommand{\tabref}[1]{Tab.~\ref{tab:#1}}

\usepackage{acro}
\usepackage{acro/acro}

\usepackage{siunitx}[=v2] % v2 needed for \units
\usepackage{units}  % \unit[value]{si}
\usepackage[noabbrev,capitalise]{cleveref} % should be last package: \cref{}

\usepackage[caption=false,font=footnotesize]{subfig}

\usepackage{ctable}
\title{
Real-Time Initialization of Unknown Anchors for UWB-aided Navigation
}

\author{% <-this % stops a space
Giulio Delama$^{1}$, Igor Borowski$^{1}$, Roland Jung$^{1}$ and Stephan Weiss$^{1}$
\thanks{$^{1}$Giulio Delama, Igor Borowski, Roland Jung, and Stephan Weiss are with the Control of Networked Systems Group, University of Klagenfurt, Austria. {\tt \{giulio.delama,igor.borowski,roland.jung, stephan.weiss\}@ieee.org}}
\thanks{This work was supported by the Federal Ministry for Climate Action,
Environment, Energy, Mobility, Innovation and Technology (BMK) under the grant agreement 894790 (SALTO).}
\thanks{\textbf{{Preprint version~\copyright IEEE, DOI: 10.1109/IROS60139.2025.11247190.}}}
}

\hypersetup{
  pdfinfo={
    Title={},
    Author={},
    Subject={},
    Keywords={}
  }
}

\begin{document}
\bstctlcite{BSTcontrol}

\maketitle
\thispagestyle{empty}
\pagestyle{empty}

%%%%%%%%%%%%%%%%%%%%%%%%%%%%%%%%%%%%%%%%%%%%%%%%%%%%%%%%%%%%%%%%%%%%%%%%%%%%%%%%

\begin{abstract}
This paper presents a framework for the real-time initialization of unknown Ultra-Wideband (UWB) anchors in UWB-aided navigation systems.
The method is designed for localization solutions where UWB modules act as supplementary sensors.
Our approach enables the automatic detection and calibration of previously unknown anchors during operation, removing the need for manual setup.
By combining an online Positional Dilution of Precision (PDOP) estimation, a lightweight outlier detection method, and an adaptive robust kernel for non-linear optimization, our approach significantly improves robustness and suitability for real-world applications compared to state-of-the-art. In particular, we show that our metric which triggers an initialization decision is more conservative than current ones commonly based on initial linear or non-linear initialization guesses. This allows for better initialization geometry and subsequently lower initialization errors.
We demonstrate the proposed approach on two different mobile robots: an autonomous forklift and a quadcopter equipped with a UWB-aided Visual-Inertial Odometry (VIO) framework.
The results highlight the effectiveness of the proposed method with robust initialization and low positioning error.
We open-source our code in a C++ library including a ROS wrapper.
\end{abstract}
\section{Introduction}
\label{sec:intro}
Autonomous systems often face challenges in Global Navigation Satellite System (GNSS)-denied environments.
In such scenarios, alternative localization methods become critical.
One promising solution is integrating Ultra-Wideband (UWB) technology, which offers positioning capabilities by measuring distances between the system and fixed UWB anchors.
However, the overall accuracy of UWB-aided navigation systems heavily depends on the precise calibration of the anchors' position.
Typically, the operator manually conducts this calibration process in a separate phase before the actual mission begins.
While effective, this can be time-consuming and impractical in specific scenarios, such as large-scale deployments where manual calibration becomes unfeasible.
To overcome this limitation, various techniques have been studied in literature to calibrate the anchors' position leveraging the mobile robot's onboard sensors and incoming UWB range data.
%to calibrate the anchors' position automatically have been developed to overcome this limitation.
All such approaches require to decide with which and with how many range measurements the 3D UWB anchor position is initialized. Current metrics for such a decision process are often overconfident leading to poor or even degenerated initial position guesses of the anchor.
\begin{figure}[t]
    \centering
    \includegraphics[width=1.0\linewidth]{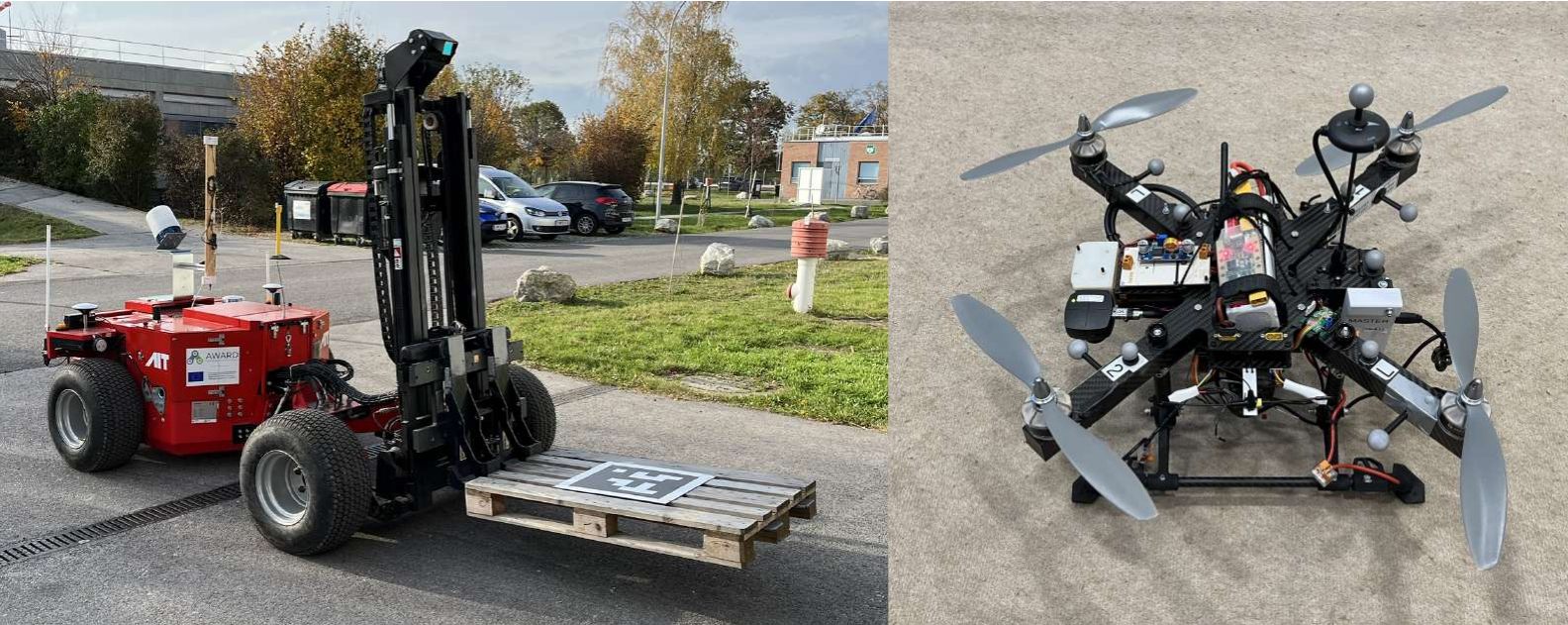}
    \vspace{-6mm}
    \caption{The two different mobile robots used in the real-world demonstration of our anchor initialization method: an automated forklift (left) and a quadcopter (right). The forklift uses two UWB receivers for 3D anchor positioning, as a single tag cannot provide vertical axis information for a ground vehicle. A single UWB receiver is sufficient for the quadcopter, assuming movement along the vertical axis.}
    \label{fig:robots}
    \vspace{-5mm}
\end{figure}

This paper presents a framework with a novel initialization decision process to automatically detect and reliably estimate the position of previously unknown UWB anchors for UWB-aided navigation.
%Our method enables the system to perform accurate real-time initialization of anchors without requiring a separate calibration phase, thus making it more suitable for practical applications.
During operation, the system collects measurements from available UWB anchors until a geometric configuration that ensures accurate anchor position estimation is identified.
A simple yet effective online outlier rejection method ensures data integrity, even for compute-constrained platforms. 
We propose a novel approach based on estimating the Positional Dilution of Precision (PDOP) in real-time to determine whether the geometric configuration of the data from a given anchor is sufficient for precise initialization.
%Once this criterion is satisfied, an initial estimate of the anchor's position is obtained by solving a Least Squares (LS) problem using the collected UWB range and system position data.
%This initial estimate is then refined through an adaptive robust-kernel Nonlinear Least Squares (NLS) optimization algorithm~\cite{Chebrolu2021AdaptiveProblems} that mitigates the impact of remaining outliers and addresses varying noise characteristics that may arise in real-world setups.
Once the criterion is met, the anchor's position is initially estimated using Least Squares (LS) and then refined through an adaptive robust-kernel Nonlinear Least Squares (NLS) method~\cite{Chebrolu2021AdaptiveProblems} to mitigate residual outliers and noise variations.
After initialization, range measurements from the calibrated anchor can be integrated into the navigation framework, enabling early drift compensation and improving accuracy and robustness.

The proposed method is validated through simulations and real-world experiments, proving its robustness to poor initializations and accurate estimation of unknown anchors.
We demonstrate the effectiveness of our solution with two different mobile robots: an Autonomous Mobile Robot (AMR) and an Unmanned Aerial Vehicle (UAV), as shown in~\figref{robots}.
The AMR combines GPS and wheel odometry for localization, while the UAV employs an open-source UWB-aided Visual-Inertial Odometry (VIO) framework UVIO~\cite{Delama2023UVIO:Initialization}.

The main contributions of this work are summarized:
\begin{itemize}
    \item Fully automated and easily integrable framework for the robust online initialization of unknown anchors' position for UWB-aided navigation solutions.
    \item Novel real-time PDOP estimation based on the \emph{closest-point-to-anchor} to trigger the initialization when the geometric configuration ensures accurate anchor positioning.
    \item Extensive validation of our framework through simulations and real-world experiments, including real-time demonstrations with two different robotic platforms: an AMR (automated forklift) and a UAV (quadcopter) utilizing an UWB-aided VIO framework\footnote{https://github.com/aau-cns/uvio.git}.
    \item Implementation and open-sourcing of the code in a publicly available C++ library\footnote{https://github.com/aau-cns/uwb\_init.git} with a ROS wrapper.
\end{itemize}
\section{Related Work}
\label{sec:related}
%The autonomous calibration of UWB anchors' position is a topic of significant research interest,  with a large body of studies addressing this challenge.
The automated calibration of UWB anchors' position is a subject of significant research interest, with numerous studies tackling this issue.
Several studies have approached the problem by assuming that anchor-to-anchor measurements are available~\cite{Pelka2016IterativeSystems, Krapez2020AnchorSystems, Nguyen2021VIRALSLAM, Nguyen2022VIRAL-Fusion:Approach, Herbruggen2023MultihopPositioning, Ridolfi2021UWBApproach, Corbalan2023Self-LocalizationPractice, Qi2024CalibrationLocalization, Mahmoud2022Ultra-widebandTracking}, often focusing on scenarios with a dense grid of UWB devices.
In contrast, our approach does not rely on anchor-to-anchor but only on tag-to-anchor measurements, relaxing the assumptions and making it suited for scenarios with fewer anchors that support other localization methods.

Studies that consider only tag-to-anchor measurements are as follows.
Initial research used linear~\cite{Hausman2016Self-calibratingUAV, Batstone2017TowardsAnchors} or iterative~\cite{Shi2019AnchorMeasurements} least squares methods to solve the initialization problem in a single step, neglecting the influence of the problem's geometry on the accuracy of the estimation.
More recent work~\cite{Nguyen2020Tightly-CoupledSystem, Nguyen2021Range-FocusedLocalization, Gao2022LowLocalization} introduces nonlinear optimization and a trajectory-variance-based initialization~\cite{Nguyen2021Range-FocusedLocalization}, but they require an initial anchor estimate.
Trajectory variance is also less reliable than the dilution of precision used in other methodologies.

The authors of~\cite{Jia2022FEJ-VIRO:Odometry} propose a long-short time-window procedure for anchor initialization using nonlinear optimization.
However, this method relies on a fixed duration for initialization data collection, limiting its effectiveness.
Their subsequent work~\cite{Jia2023DistributedNetwork} introduces a distributed method requiring dedicated computational units for the anchors.
In contrast, our approach reduces assumptions, considers the trajectory's geometry, and does not require specialized computational units.

A recent paper~\cite{Li2023UWB-VO:Odometry} employs a factor graph for UWB-aided monocular-vision SLAM with a single anchor, and another study~\cite{Hamesse2024FastSystem} uses a similar approach to integrate LiDAR-inertial SLAM with UWB measurements in an indoor environment.
A different approach~\cite{Hu2023TightlyAnchors} uses a deep neural network to visually detect UWB modules and support initialization.
Apart from relying on specific additional sensors, these approaches can be limited by the high computational cost and reduced generality.

The above-mentioned work does not consider how initialization trajectory geometry affects anchor calibration.
This aspect is examined next.
In a recent study~\cite{Blueml2021BiasPoints}, the authors applied Fisher Information Matrix (FIM) principles to optimize the selection of data collection points for initializing a single UWB anchor with a UAV, following an initial coarse triangulation with random vehicle positions.
\begin{figure}[t]
    \centering
    \includegraphics[width=1.0\linewidth]{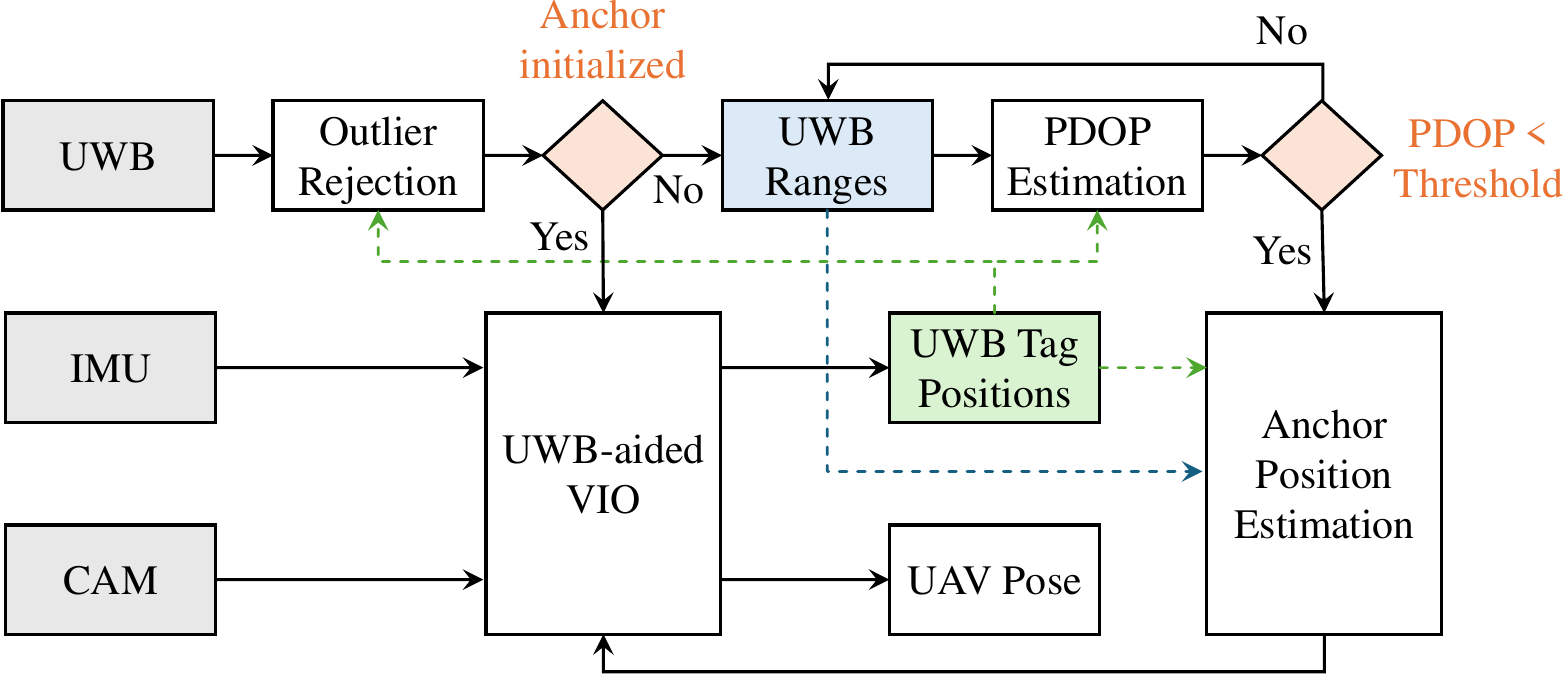}
    \vspace{-5mm}
    \caption{Diagram illustrating the proposed framework for initializing unknown UWB anchors within a UWB-aided VIO system. A key feature is the real-time PDOP estimation, which triggers anchor initialization once it drops below a defined threshold, preventing poor initialization due to unfavorable geometric configurations. This approach allows for the continuous detection and initialization of new anchors becoming available during operation.}
    \label{fig:flowchart}
    \vspace{-5mm}
\end{figure}
Subsequent work~\cite{Delama2023UVIO:Initialization} builds upon this concept and introduces the PDOP as a metric to improve the anchors' initialization.
Specifically, an optimal set of waypoints is computed that minimizes the PDOP for the specific anchors' arrangement, enhancing the geometric configuration between the UAV and multiple anchors simultaneously.
Like~\cite{Blueml2021BiasPoints}, this approach requires initial position estimates for the anchors from a prior random flight.

Similarly,~\cite{Hu2023RobustUAVs} proposes a two-step approach for UWB anchor initialization, incorporating optimal path planning.
Despite the potential of these methods, they have an intrinsic limitation: they require the mobile robot to follow specific waypoints in a dedicated initialization phase, which may not always be feasible in real-world scenarios.
This drastically limits their practical appeal, particularly for ground vehicles.

In a follow-up study~\cite{Jung2024ModularCalibration}, the authors show the effectiveness of the method presented in~\cite{Delama2023UVIO:Initialization} by extending it to multi-tag configurations.
This extension allows for measurements between known and unknown anchors and includes a RANSAC-based outlier rejection technique to enhance its robustness.
Nonetheless, in their assessment of the online anchor initialization approach, the calibration process for a batch of anchors was initiated when a timeout event occurred, which does not guarantee a geometric quality for initialization.

A newly published article~\cite{Luo2025Visual-inertialPosition} employs a robust ridge NLS algorithm to estimate the anchors' position while the PDOP is used to determine if the initialization is sufficiently reliable.
Similarly to~\cite{Luo2025Visual-inertialPosition}, our method overcomes the limit of~\cite{Blueml2021BiasPoints, Delama2023UVIO:Initialization, Hu2023RobustUAVs} and removes the need for a separate initialization phase, enhancing the efficiency of UAV inspection and AMR transport missions.
Moreover, unlike~\cite{Luo2025Visual-inertialPosition} and other geometric methods~\cite{Blueml2021BiasPoints, Delama2023UVIO:Initialization, Hu2023RobustUAVs}, which evaluate the geometry only after initialization using the estimated anchor positions, we propose a \emph{real-time PDOP computation that is provably underconfident}.
This allows us to perform anchor initialization only when the geometric configuration ensures high positioning accuracy.

As a result, our approach is not only simpler to use but also more computationally efficient, avoiding unnecessary anchor estimation steps while ensuring robust initialization performance.
Note that some studies may use the term Geometric Dilution of Precision (GDOP) interchangeably with PDOP.
However, GDOP is specifically related to GNSS navigation and includes a time component. 
In this context, using PDOP as the term would be more accurate and avoid confusion.
\section{Robust Initialization of UWB Anchors}
\label{sec:method}
This section outlines the proposed method for initializing previously unknown anchors in UWB-aided navigation systems during operation.
We focus on a scenario involving passive UWB anchors, where only tag-to-anchor measurements and no anchor-to-anchor data are available.
The full framework is illustrated in the diagram of~\figref{flowchart} serving as a reference throughout this section.

\subsection{Problem Description}
\begin{figure}[b]
    \centering
    \vspace{-5mm}
    \includegraphics[width=1.0\linewidth]{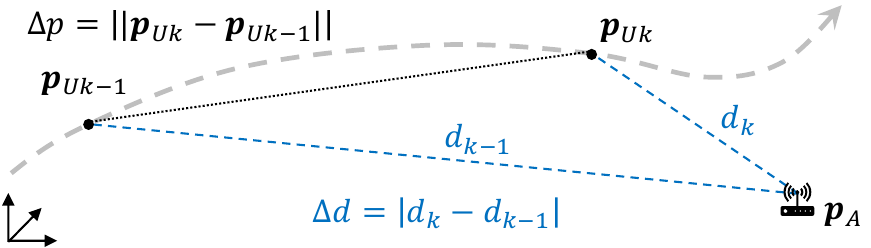}
    \vspace{-5mm}
    \caption{Outlier rejection method using an online consistency check. The absolute differences between consecutive range measurements $\Delta d$ and tag positions $\Delta p$ are compared. Outliers are identified and rejected when the condition ${\Delta d < \Delta p + \tau}$ is not met, with $\tau$ being an adjustable threshold.}
    \label{fig:outlier_rejection}
    \vspace{-5mm}
\end{figure}
Consider a mobile robot equipped with a UWB tag delivering range measurements from one or multiple UWB anchors.
For each anchor $A_i$, the distance measurement is modeled as
\begin{equation}
    d_i(t) = \norm{\Vector{}{p}{U} - \Vector{}{p}{A_i}} + \gamma_i + \eta_i ,
    \label{equ:meas}
\end{equation}
where ${\Vector{}{p}{U} \in \Rn{3}}$ denotes the position of the UWB tag and ${\Vector{}{p}{A_i} \in \Rn{3}}$ the position of the $i$-th UWB anchor in the world frame, while ${\gamma_i \in \R}$ represents a constant bias term and ${\eta_i \sim \mathcal{N}(0,\,\sigma_d^2)}$ denotes the measurement noise.
The UWB tag position $\Vector{}{p}{U}$  can be computed based on the robot’s 6-DoF pose and the rigid transformation defining the tag’s extrinsic calibration.
Assume that the robot operates in an environment where one or multiple UWB anchors have been placed in fixed but unknown positions.
Additional sensors, e.g., GPS, IMU, camera, etc., allow for estimating the robot’s position and orientation at least initially until the first UWB anchors are initialized and can be used for subsequent navigation.
For the remainder of the section, without loss of generality, we focus on the problem of initializing a single UWB anchor, allowing us to simplify the notation by dropping the subscript $i$ from range measurements and anchor position.
Under these assumptions, the data available to the system for initializing each anchor consists of the range measurements $d_{k}$, collected at timestamps $t_k$, and UWB tag positions ${\Vector{}{p}{U_k} = \Vector{}{p}{U}(t_k)}$ interpolated at each UWB timestamp.

\subsection{Outlier Rejection}
\label{sec:outlier}
In UWB-ranging systems, outliers are a common issue, particularly in complex environments with multipath effects and non-line-of-sight (NLOS) conditions frequently encountered in real-world scenarios.
These outliers can significantly degrade the accuracy of the anchors' position estimation.
Therefore, fast and robust outlier rejection is crucial for achieving reliable localization.
A very recent study~\cite{Sun2024AEnvironments} introduces a \emph{triangle-rule consistency check} to identify outliers in their 2D scenario.
Our approach extends this simple but effective idea to 3D applications, and it is based on comparing two consecutive UWB ranges and poses as shown in~\figref{outlier_rejection}.

Let $d_{k-1}$ and $d_k$ be two consecutive distance measurements and $\Vector{}{p}{U_{k-1}}$ and $\Vector{}{p}{U_k}$ the corresponding tag positions.
By defining the quantities ${\Delta p = \norm{\Vector{}{p}{U_k} - \Vector{}{p}{U_{k-1}}}}$ and ${\Delta d = |d_k - d_{k-1}|}$ we can implement the consistency check which should satisfy ${\Delta d \leq \Delta p + \tau}$, with ${\tau \in \R}$ being an adjustable threshold. If, e.g., the measurement noise follows a Gaussian distribution, $\tau$ can be set to a multiple of the standard deviation to act as a hypothesis check. 
Any UWB measurement $d_k$ that fails to meet this condition is classified as an outlier and discarded.
The proposed method reliably detects outliers while remaining computationally efficient, making it well-suited for real-time applications.
\autoref{fig:filtered_data} illustrates the effectiveness of the proposed method applied to real-world UWB range data.
\begin{figure}[t]
    \centering
    \includegraphics[width=1.0\linewidth]{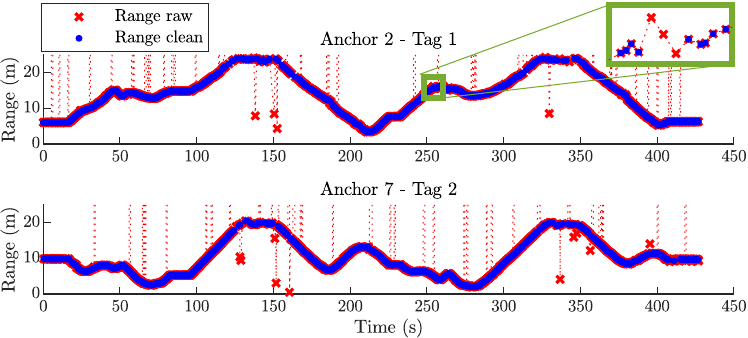}
    \vspace{-7mm}
    \caption{Outlier rejection performance in a real-world experiment with an automated forklift. This figure shows the effectiveness of the proposed online outlier rejection method applied to real-world UWB range data. Red dashed lines indicate outliers that fall outside the plot range. The green box provides a sampled zoomed-in view, highlighting that smaller outliers are also detected.}
    \label{fig:filtered_data}
    \vspace{-5mm}
\end{figure}

\subsection{Real-time PDOP Estimation}
\label{sec:pdop}
Accurate anchor calibration requires evaluating the geometry of the initialization trajectory relative to the actual anchor position.
The PDOP is a key metric for estimating the quality of this geometric configuration~\cite{Fontanelli2021Cramer-RaoG-WLS}.
It is calculated as follows:
\begin{equation}\label{equ:pdop}
    \mathrm{PDOP} = \sqrt{\tr{(\mathbf{G}^\top \mathbf{G})^{-1}}}
\end{equation}
where, for $N$ range measurements $d_{1, 2,\ldots, N}$, ${\mathbf{G} \in \Rn{N \times 3}}$ is
\begin{equation}\label{equ:gmatrix}
    \!\!\!\!\mathbf{G}\!=\!\begin{bmatrix}
    \frac{(\Vector{}{p}{U_1} - \Vector{}{p}{A})^\top}{d_1} \\
    \frac{(\Vector{}{p}{U_2} - \Vector{}{p}{A})^\top}{d_2} \\
    \vdots \\
    \frac{(\Vector{}{p}{U_N} - \Vector{}{p}{A})^\top}{d_N}
    \end{bmatrix}\!=\!\begin{bmatrix}
    \frac{p_{U_{1x}} - p_{A_x}}{d_1} & \frac{p_{U_{1y}} - p_{A_y}}{d_1} & \frac{p_{U_{1z}} - p_{A_z}}{d_1}\\
    \frac{p_{U_{2x}} - p_{A_x}}{d_2} & \frac{p_{U_{2y}} - p_{A_y}}{d_2} & \frac{p_{U_{2z}} - p_{A_z}}{d_2}\\
    \vdots &\vdots &\vdots\\
    \frac{p_{U_{Nx}} - p_{A_x}}{d_N} & \frac{p_{U_{Ny}} - p_{A_y}}{d_N} & \frac{p_{U_{Nz}} - p_{A_z}}{d_N}
    \end{bmatrix} .
\end{equation}
A lower PDOP value implies a better geometric configuration for precise calibration.
Note that the PDOP computation requires not only the measured distances $d_{1, 2,\ldots, N}$ and the known UWB tag positions $\Vector{}{p}{U_{1, 2,\ldots, N}}$, but also the anchor position $\Vector{}{p}{A}$ which is unknown by assumption in our problem formulation.

In~\cite{Luo2025Visual-inertialPosition}, PDOP is used to assess estimation accuracy, and it is computed \emph{after estimating an initial anchor position}.
However, this approach has an inherent limitation: the estimated PDOP could be overconfident due to its dependency on the accuracy of the initially calculated anchor position, which can experience large fluctuations in the beginning or in scenarios with poor geometric configurations.
\begin{figure}[t]
    \centering
    \includegraphics[width=1.0\linewidth]{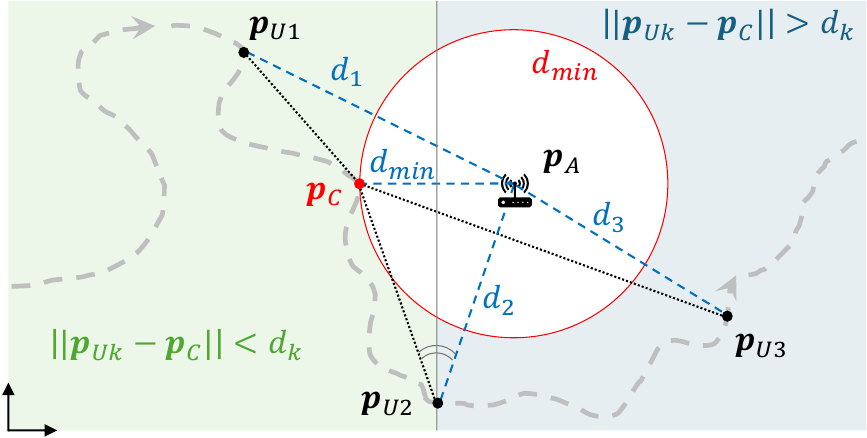}
    \vspace{-7mm}
    \caption{Geometric representation of initialization regions and their influence on our closest-point-to-anchor PDOP estimation from \eqref{equ:gtilde}.
    When initialization occurs in the green zone, which is the typical case, e.g., when anchors are attached to walls, the PDOP estimation remains conservative.
    If initialization takes place in the blue zone, it means that the trajectory "wraps around" the anchor, and thus, the PDOP is inherently low due to good geometric condition, rendering our then slight overconfidence in the metric negligible.}
    \label{fig:pdop}
    \vspace{-4mm}
\end{figure}
To address this problem, we introduce a real-time PDOP estimation method as initialization trigger that is \emph{intrinsically underconfident and does not require prior knowledge of an initial anchor position}.
Our approach leverages the \emph{closest-point-to-anchor} $\Vector{}{p}{C}$, which is defined as the known tag position $\Vector{}{p}{U_j}$ corresponding to the shortest range measurement ${d_j = d_{min} = \min\{d_1, d_2, ..., d_N\}}$.
To do so, we modify the matrix $\mathbf{G}$ in \eqref{equ:gmatrix} and define a new matrix $\tilde{\mathbf{G}}$ as
\vspace{-1mm}
\begin{equation}\label{equ:gtilde}
    \!\!\!\!\tilde{\mathbf{G}}\!=\!\begin{bmatrix}
    \vdots \\
    \frac{(\Vector{}{p}{U_k} - \Vector{}{p}{C})^\top}{d_k} \\
    \vdots
    \end{bmatrix} \in \Rn{(N-1)\times3} ,
\vspace{-1mm}
\end{equation}
where the row corresponding to $\Vector{}{p}{U_k} = \Vector{}{p}{U_j}$ is removed.
The conservative \emph{closest-point-to-anchor} PDOP is then computed as in \eqref{equ:pdop}, replacing $\mathbf{G}$ with the modified geometry matrix $\tilde{\mathbf{G}}$.

\begin{theorem}\label{thm:conservative_pdop}
Let $\mathbf{G} \in \mathbb{R}^{N \times 3}$ be the geometry matrix defined in \eqref{equ:gmatrix}. Let ${\tilde{\mathbf{G}} \in \mathbb{R}^{(N-1) \times 3}}$ be the modified geometry matrix \eqref{equ:gtilde}, obtained by replacing the unknown anchor position $\Vector{}{p}{A}$ with its closest known tag position $\Vector{}{p}{C}$. Then, the closest-point-to-anchor PDOP estimated using $\tilde{\mathbf{G}}$ provides a conservative upper bound on the true PDOP if ${\norm{\Vector{}{p}{U_k} - \Vector{}{p}{C}} \leq \norm{\Vector{}{p}{U_k} - \Vector{}{p}{A}} \forall k}$.
\end{theorem}

\begin{proof}
By expanding the information matrix as a sum of outer products, one obtains
\vspace{-1mm}
\begin{equation}
    \mathbf{G}^\top \mathbf{G} = \sum_{k=1}^{N} \frac{(\Vector{}{p}{U_k} - \Vector{}{p}{A})(\Vector{}{p}{U_k} - \Vector{}{p}{A})^\top}{d_k^2}.
\end{equation}
For the modified geometry matrix $\tilde{\mathbf{G}}$ and ${\Vector{}{p}{C} = \Vector{}{p}{j}}$, we have
\begin{equation}
    \tilde{\mathbf{G}}^\top \tilde{\mathbf{G}} = \sum_{\substack{k=1 \\ k \neq j}}^{N} \frac{(\Vector{}{p}{U_k} - \Vector{}{p}{C})(\Vector{}{p}{U_k} - \Vector{}{p}{C})^\top}{d_k^2}.
\end{equation}
By assumption we have ${\|\Vector{}{p}{U_k} - \Vector{}{p}{C}\| \leq \|\Vector{}{p}{U_k} - \Vector{}{p}{A}\| = d_k \forall k}$.
Dividing by $d_k$ implies that the normalized vectors in $\tilde{\mathbf{G}}$ satisfy
\begin{equation}
    \frac{\|\Vector{}{p}{U_k} - \Vector{}{p}{C}\|}{d_k} \leq \frac{\|\Vector{}{p}{U_k} - \Vector{}{p}{A}\|}{d_k} = 1.
\end{equation}
Since the terms in $\tilde{\mathbf{G}}^\top \tilde{\mathbf{G}}$ contain the outer product of these scaled vectors, each term in the sum contributes a smaller or equal magnitude compared to the corresponding term in $\mathbf{G}^\top \mathbf{G}$. Thus, we have the inequality ${\tilde{\mathbf{G}}^\top \tilde{\mathbf{G}} \preceq \mathbf{G}^\top \mathbf{G}}$ and, by the monotonicity of the matrix inverse, we get ${(\tilde{\mathbf{G}}^\top \tilde{\mathbf{G}})^{-1} \succeq (\mathbf{G}^\top \mathbf{G})^{-1}}$.
Taking the square root of the trace and applying its monotonicity over positive semidefinite matrices, we obtain
\begin{equation}
    \sqrt{\tr{(\tilde{\mathbf{G}}^\top \tilde{\mathbf{G}})^{-1}}} \geq \sqrt{\tr{(\mathbf{G}^\top \mathbf{G})^{-1}}}.
\end{equation}
This proves that if ${\norm{\Vector{}{p}{U_k} - \Vector{}{p}{C}} \leq \norm{\Vector{}{p}{U_k} - \Vector{}{p}{A}} \forall k}$ then the estimated PDOP using the closest known tag position is always greater or equal to the true PDOP.
\end{proof}

\autoref{fig:pdop} illustrates a simplified 2D representation of the geometry constraints for the PDOP estimation problem and the Theorem \ref{thm:conservative_pdop}.
The distance inequality states that as long as the distance from any point on the trajectory where a measurement was taken for the PDOP computation to the point on the trajectory closest to the anchor (tag measurement with minimal distance measurement in the set), i.e., $\norm{\Vector{}{p}{U_k} - \Vector{}{p}{C}}$, is smaller than the distance from any of those points to the anchor, i.e., $\norm{\Vector{}{p}{U_k} - \Vector{}{p}{A}}$, our approach is inherently conservative.
Additionally, our method demonstrates robustness during the critical phase when initial measurements from the anchor are first received, even when the anchor is far away.
As shown in~\figref{pdop_sim}, simulation results confirm that the PDOP is always greater than or equal to the true PDOP and exhibits a strong correlation with the average initialization error.
In practical implementation, initialization is triggered when the online estimated PDOP falls below a predefined threshold that depends on the specific use case (see \secref{results}).
\begin{figure}[t]
    \centering
    \includegraphics[width=1.0\linewidth]{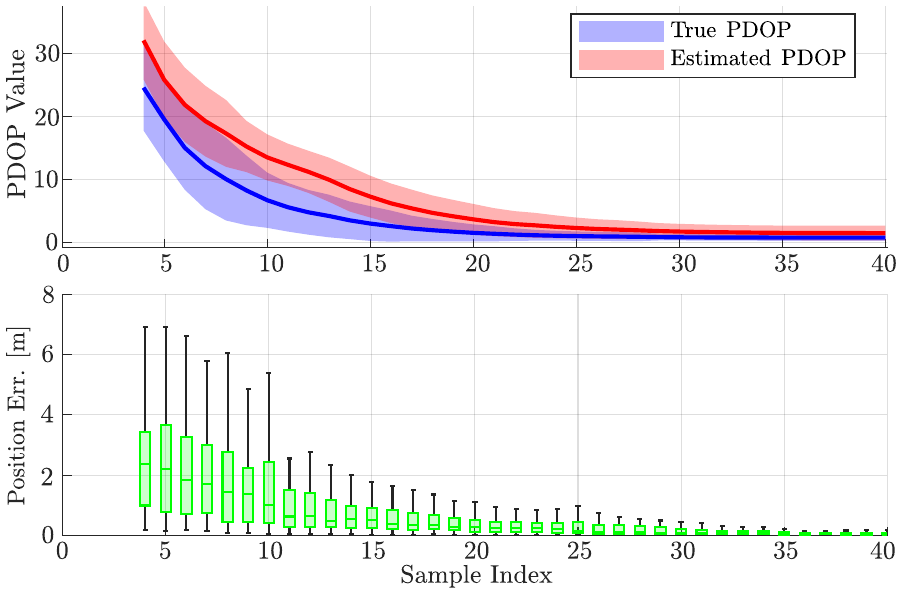}
    \vspace{-7mm}
    \caption{Comparison between true and our estimated PDOP (top) and corresponding initialization error (bottom). The figure shows the results of a Monte Carlo (MC) simulation with $M = 10^2$ different random trajectories and measurement noise realizations (${\sigma_d = 0.15~m}$). The top plot reveals that the estimated PDOP is consistently higher than the true one, while the bottom plot shows a correlation between the PDOP and the initialization error. Sample index refers to the increasing number of data samples used for initialization.}
    \label{fig:pdop_sim}
    \vspace{-4mm}
\end{figure}

\subsection{Anchor Position Estimation}
When the PDOP criterion is satisfied, anchor initialization begins with a \emph{coarse solution} obtained by solving a total LS problem using the \emph{Optimal Double Method} detailed in~\cite{Delama2023UVIO:Initialization}.
Specifically we get ${\mathbf{A} \Vector{}{x}{} = \Vector{}{b}{}}$ with ${\Vector{}{x}{} = \begin{bmatrix} \Vector{}{p}{A} & \gamma \end{bmatrix}^T \in \Rn{4 \times 1}}$ including the anchor position $\Vector{}{p}{A}$ and constant bias $\gamma$ from \eqref{equ:meas}.
Each $k$-th row of ${\mathbf{A} \in \Rn{(N-1) \times 4}}$ is computed as
\begin{equation}
    \mathbf{A}_k = \begin{bmatrix}
        -\left({\Vector{}{p}{U_k}} - {\Vector{}{p}{U_j}}\right)^T & \left(d_k - d_j\right)
    \end{bmatrix},
\end{equation}
and each $k$-th row of ${\Vector{}{b}{} \in \Rn{(N-1) \times 1}}$ is computed as
\begin{equation}
    b_k = 
        \frac{1}{2}\left(\left(d_k^2 - d_j^2\right) - \left(\norm{\Vector{}{p}{U_k}}^2 - \norm{\Vector{}{p}{U_j}}^2\right)\right) .
\end{equation}
Note that the index $j$ can be selected arbitrarily and does not necessarily correspond to the closest point $\Vector{}{p}{C}$.
In~\cite{Delama2023UVIO:Initialization}, a method is proposed for selecting the index that minimizes the uncertainty in the LS solution.
However, the choice of this \emph{pivot} index has minimal impact on the final solution, as a non-linear optimization algorithm is then applied to the initialization data with the LS solution as the initial guess.
The standard NLS problem aims at finding the parameter vector ${\Vector{}{\theta}{} = (\Vector{}{p}{A}, \gamma) \in \Rn{4}}$ that minimizes a sum of squared residuals:
\begin{equation}\label{equ:NLS}
    \Vector{}{\barVector{}{\theta}{}}{} = \min_{\Vector{}{\theta}{}} \sum_{k=1}^{N} \norm{r_k(\Vector{}{\theta}{})}^2 = \min_{\Vector{}{\theta}{}} \sum_{k=1}^{N} \norm{\hat{d}_k(\Vector{}{\theta}{}) - d_k}^2 ,
\end{equation}
where ${r_k(\Vector{}{\theta}{}) = \hat{d}_k(\Vector{}{\theta}{}) - d_k}$ represents the $k$-th residual between a prediction ${\hat{d}_k(\Vector{}{\theta}{}) = d(\Vector{}{\theta}{},t_k)}$ from \eqref{equ:meas} and the actual UWB range measurement $d_k$. 
In~\cite{Delama2023UVIO:Initialization}, the Levenberg-Marquardt (LM) algorithm is used for optimization, but the presence of noise, remaining outliers, and unmodeled effects in the UWB-ranging data can compromise the accuracy of the final result.
Studies like~\cite{Batstone2017TowardsAnchors, Hamesse2024FastSystem} demonstrated the effectiveness of \emph{robust kernels} in down-weighting large UWB residuals during non-linear optimization. 
However, they require appropriate kernel selection and parameter tuning.
A similar limitation is found in~\cite{Jung2024ModularCalibration}, which includes a RANSAC-based outlier rejection.

To overcome this problem and increase our system's robustness, in addition to the outlier rejection method introduced in \secref{outlier}, we incorporate the \emph{Adaptive Robust Kernel}~\cite{Chebrolu2021AdaptiveProblems} that automatically adapts to the distribution of the residuals.
The new minimization problem is formulated as follows:
\begin{equation}
    (\barVector{}{\theta}{},\bar{\alpha}) = \min_{\Vector{}{\theta}{},\alpha} \sum_{k=1}^{N}\rho(r_k(\Vector{}{\theta}{}),\alpha) ,
    \label{equ:robust}
\end{equation}
where $\rho(r_k(\Vector{}{\theta}{}),\alpha)$ is the \emph{generalized robust loss function}
\begin{equation}
    \rho(r,\alpha,c) = \frac{|\alpha - 2|}{\alpha}\left(\left(\frac{(r/c)^2}{|\alpha - 2|}+1\right)^{\alpha/2}-1\right) ,
\end{equation}
${\alpha \in \R}$ is an additional optimization parameter that controls the shape of the kernel, and $c \in \R$ is a fixed scale parameter.
The full description of the algorithm and how to solve the minimization problem \eqref{equ:robust} is omitted due to space limitations but is comprehensively explained in~\cite{Chebrolu2021AdaptiveProblems}.
This approach enables the joint estimation of both kernel shape $\alpha$ and parameter vector $\Vector{}{\theta}{}$, accounting for varying noise levels and improving robustness against outliers and unmodeled environmental effects.

\section{Experiments and Results}
\label{sec:results}
The method outlined in the previous section is validated with simulated and real-world experiments.

\subsection{Simulation}
\begin{table}[t]
    \centering
    \caption{Comparison of initialization performance in simulation: PDOP-based (our) vs. fixed-window strategy}
    \begin{tabular}{|>{\centering\arraybackslash}p{0.2cm}|>{\centering\arraybackslash}p{0.8cm}|>{\centering\arraybackslash}p{1.2cm}|>{\centering\arraybackslash}p{1.2cm}|>{\centering\arraybackslash}p{0.6cm}|>{\centering\arraybackslash}p{0.6cm}|>{\centering\arraybackslash}p{1.2cm}|}
    \hline
          \# & Method & Avg. ($m$) & Med. ($m$) & Init. & >1$m$ & Ratio ($\%$) \\ \hline\hline
          \multirow {2}{*}{1}
          &Fixed & 0.233 & \textbf{0.039} & 3000 & 283 & 9.43 \\
          &Our & \textbf{0.131} & 0.054 & 1362 & 37 & \textbf{2.71} \\ \hline
          \multirow {2}{*}{2}
          &Fixed & 0.350 & \textbf{0.043} & 3000 & 505 & 16.8\\ 
          &Our & \textbf{0.100} & 0.044 & 576 & 14 & \textbf{2.43} \\ \hline
          \multirow {2}{*}{3}
          &Fixed & 0.516 & 0.055 & 3000 & 615 & 20.5 \\
          &Our & \textbf{0.148} & \textbf{0.048} & 513 & 21 & \textbf{4.09} \\ \hline
          \multirow {2}{*}{4}
          &Fixed & 0.747 & 0.229 & 3000 & 621 & 20.7 \\
          &Our & \textbf{0.343} & \textbf{0.173} & 623 & 38 & \textbf{6.09}\\ \hline
    \end{tabular}
    \label{tab:simerrors}
    \vspace{-3mm}
\end{table}
We simulate UWB tag trajectories in a tunnel environment with multiple anchors distributed along its entire length, and we generate synthetic UWB range data.
We perform multiple Monte Carlo (MC) simulations of the initialization procedure, utilizing random trajectories with varying noise levels and outliers.
Each simulation consists of $M = 10^2$ different noise realizations.
\autoref{fig:pdop_sim} shows the results of such simulations for a single anchor, evaluating the relation between the true and the estimated PDOP and its association with the initialization error.
The data shows that the error and the PDOP are correlated, and our estimated PDOP is consistently higher than the true one.
This allows for formulating a criterion where a PDOP threshold is selected to trigger initialization ensuring high accuracy.
In our evaluation, this threshold was set to 1.

We validate the robustness of our PDOP-based approach against poor initializations by comparing it to a simple fixed-window initialization strategy that uses measurements collected throughout the whole trajectory.
Suboptimal trajectories are simulated to test our method's effectiveness in detecting unfavorable geometrical configurations and avoiding erroneous anchor initialization.
Four MC simulations are performed with increasing noise ($\sigma_d = 0.1-0.5~m$) and percentage of outliers, assessing the initialization errors for both methods, as presented in \tabref{simerrors}.
The table reports key metrics: average and median initialization errors, total initialized anchors ("Init." column), anchors with errors over $1~m$, and the ratio of these bad initializations to the total. The table allows the following important conclusions: while the fixed window approach is more often more accurate (lower median error), our approach significantly outperforms it on average, particularly under higher noise conditions. This is a direct result of our conservative initialization decision and the higher quality fluctuations for the fixed window approach (also reflected in the last three columns of the table showing number of initializations, large error counts, and ratio of bad initializations, respectively).

\subsection{Real-world AMR Experiments}
\begin{figure*}[t]
    \centering
    \includegraphics[width=1.0\linewidth]{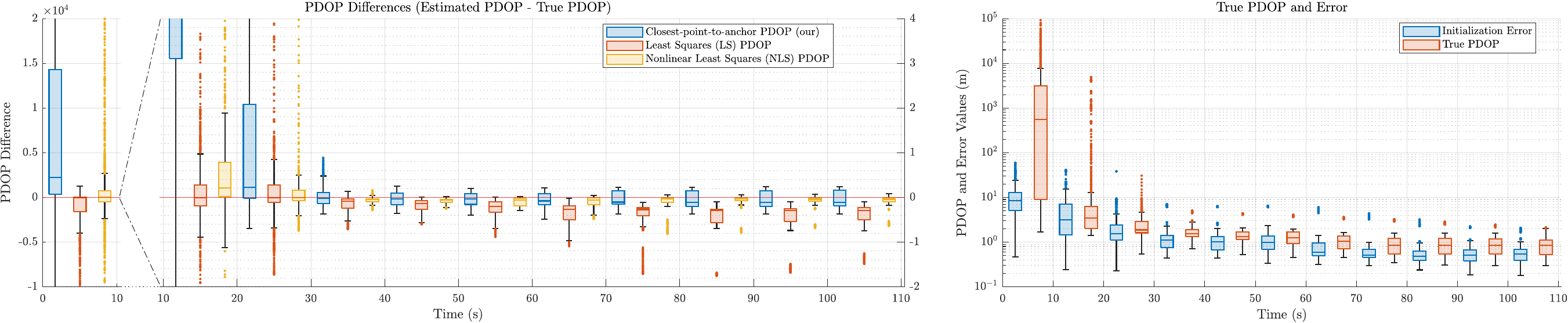}
    \vspace{-8mm}
    \caption{Boxchart comparison of PDOP estimation methods and initialization error across the first $110~s$ of the five real-world AMR trajectories in~\figref{amr_trajectories}, segmented into $10~s$ intervals to enhance statistical sampling. The left plot shows the PDOP difference (estimated - true) across various estimation methods, with a different y-axis scale for better comparison. Our closest-point-to-anchor PDOP (blue) is computationally efficient with conservative estimates (positive values here). In contrast, the LS PDOP (red) is more computationally demanding while still being real-time capable but is more unpredictable (see several overconfident, i.e., negative, dots), as its accuracy depends on the estimated anchor position, which can exhibit large fluctuations. Nonlinear LS PDOP (yellow) is more accurate but not suitable in real-time, and remains susceptible to divergence, leading to occasional outliers. In comparison, our method consistently demonstrates robust and stable performance. The right plot highlights the relationship between PDOP and initialization error on a logarithmic scale.}
    \label{fig:boxchart}
    \vspace{-5mm}
\end{figure*}
\begin{table}[t]
\centering
\caption{AMR real-world experiment: comparison of the anchors' position RMSE with different outlier rejection methods. }
%\resizebox{\linewidth}{!}{%
\begin{tabular}{|c|c|c|c|c|c|}
\hline
\#                 & Method & x ($m$)    & y ($m$)     & z ($m$)      & Avg. ($m$)  \\ \hline\hline
\multirow{2}{*}{1} & Our    & \textbf{0.17} & \textbf{0.12} & \textbf{0.35} & \textbf{0.44} \\
                   & RANSAC & \textbf{0.17} & \textbf{0.12} & 0.37          & 0.47          \\ \hline
\multirow{2}{*}{2} & Our    & \textbf{0.18} & \textbf{0.11} & 0.34          & 0.43          \\
                   & RANSAC & 0.2           & 0.13          & \textbf{0.31} & \textbf{0.42} \\ \hline
\end{tabular}
\label{tab:amr-calib}
\vspace{-3mm}
\end{table}

To evaluate our framework outdoors with real-world data collected by an AMR, we deployed in total 11 stationary UWB anchors in front of a machine hall, see left image of~\figref{robots}, with slight variation in their height ($0.9-2.9~m$).
Two UWB tags were rigidly attached to the AMR above the center of the rear axle with a horizontal displacement of $0.85~m$.
This vertical setup should facilitate the calibration of anchor heights, as the ARM operates on a nearly flat horizontal plane, typically leading to high vertical Dilution of Precision (DOP).
The UWB anchors provided measurements to each of the two tags at roughly 10Hz.
The AMR was moving in an area of roughly $40 \times 20~m^2$ as shown in~\figref{amr_trajectories}.
The true pose of the vehicle was captured by a commercial RTK GPS-INS system with absolute orientation.
The true position of the anchors was obtained by theodolite measurements, that were aligned to the RTK reference frame of the robot.

To demonstrate the effectiveness of the proposed outlier rejection method presented in \secref{outlier}, we employed the NLS slover proposed in~\cite{Delama2023UVIO:Initialization} on our prefiltered data against the RANSAC method proposed in~\cite{Jung2024ModularCalibration}.
The results listed in \tabref{amr-calib} show that the RANSAC approach generally performs well, but it is computationally expensive, and its performance greatly depends on the selected parameters.
In this evaluation, we wanted to obtain an inlier subset with a probability of ${p=95\%}$, choose ${s=60}$ measurements for the solver, and assume ${e=10\%}$ of outliers in the samples, which lead to ${n = \frac{log(1-p)}{1-(1-e)^s} = 1666}$ iterations of solving the NLS on a subset of $s$ samples.
It took, on average, 463 times longer than the proposed outlier rejection method, which removed on average $6.5\%$ of the samples, with a threshold of ${\tau=0.1 \approx 2\sigma_d}$. 

\autoref{fig:boxchart} shows an evaluation of the proposed closest-point-to-anchor PDOP introduced in \secref{pdop} over the five different trajectories of~\figref{amr_trajectories}.
Our method, i.e., the computation of the geometry matrix \eqref{equ:gtilde}, is compared against two different PDOP estimation methods commonly utilized in the literature: the LS PDOP, i.e., the computation of the geometry matrix \eqref{equ:gmatrix} is based on the LS estimate $\hatVector{}{p}{A}$, and the NLS PDOP, i.e., the computation of \eqref{equ:gmatrix} is based on the respective NLS estimate.
\begin{figure}[t]
    \centering
    \includegraphics[width=1.0\linewidth]{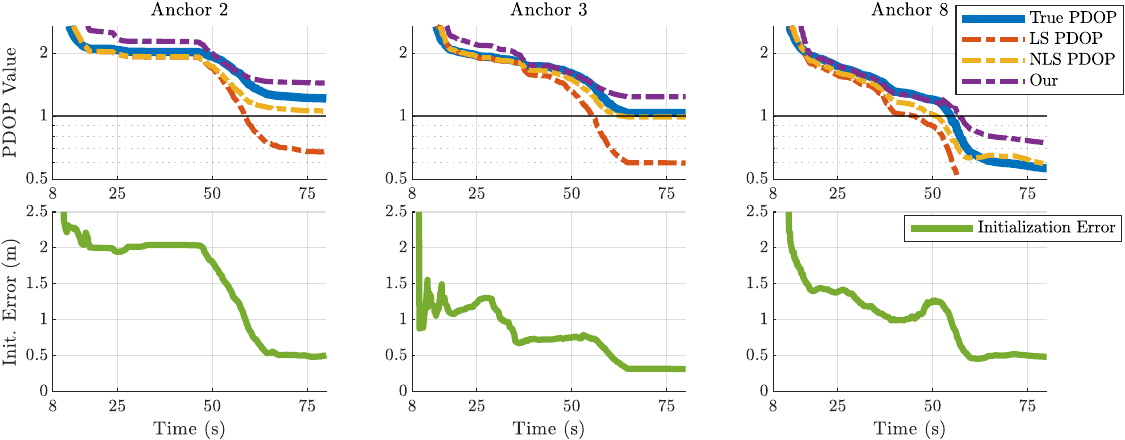}
    \vspace{-7mm}
    \caption{PDOP estimation and initialization error for three UWB anchors in real-world forklift experiments. The top plots compare PDOP values estimated by different methods against the true PDOP, with our method being the most conservative. The bottom plots show the corresponding initialization error.}
    \label{fig:amr_plot}
    \vspace{-5mm}
\end{figure}
\begin{figure}[t]
    \centering
    \includegraphics[width=1.0\linewidth]{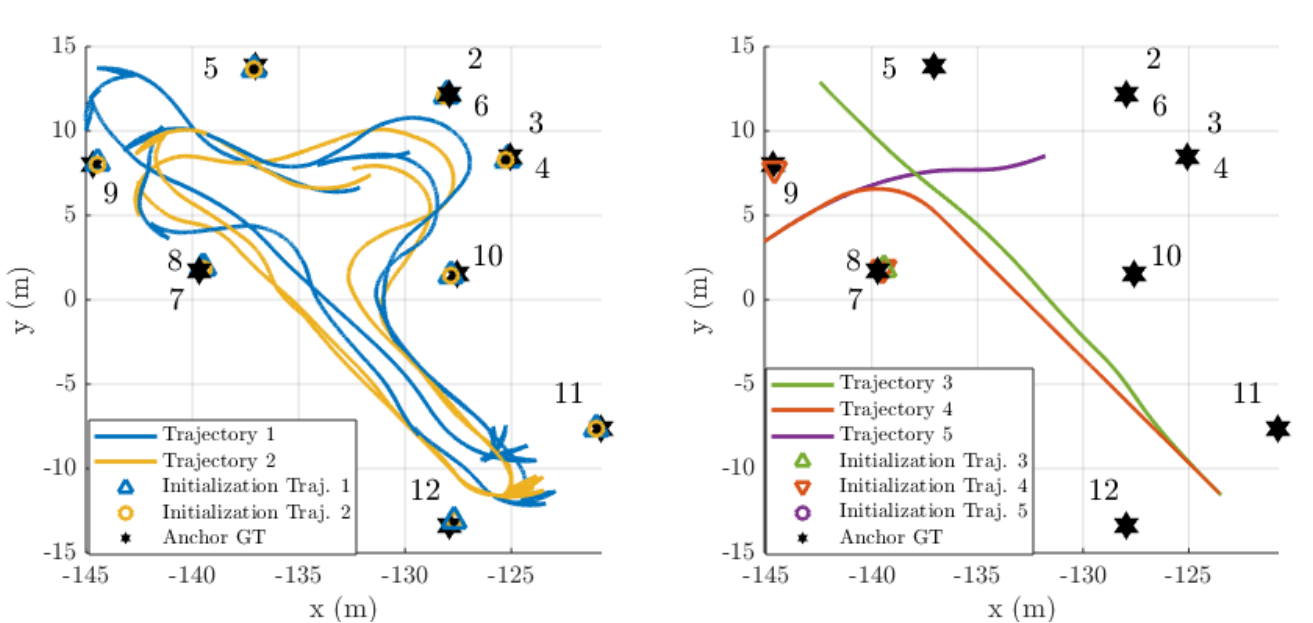}
    \vspace{-7mm}
    \caption{AMR trajectories from outdoor experiments. The left plot illustrates two trajectories (blue and yellow lines) where all 11 anchors were initialized with a PDOP threshold of 1. The average error for these initializations is less than $0.5~m$. The right plot shows three suboptimal trajectories (green, red, and purple lines). Only one anchor (7) for the green trajectory and three anchors (7, 8, 9) for the red trajectory met the PDOP threshold and were initialized, avoiding large positioning errors.}
    \label{fig:amr_trajectories}
    \vspace{-5mm}
\end{figure}

To empirically validate \cref{thm:conservative_pdop}, we computed the PDOP using solutions from different methods as the number of observations grew. \autoref{fig:amr_plot} illustrates that our proposed approach results in conservative PDOP estimates. Furthermore, the plot highlights a correlation between initialization error and PDOP, advocating for a PDOP threshold to trigger the initialization process. In scenarios with suboptimal trajectories, premature initialization can lead to significant position errors or, in the worst case, to incorrect solutions due to ambiguities. Therefore, it is crucial to avoid early initialization. \autoref{fig:amr_trajectories} displays the successful anchor initialization of our proposed approach on well-conditioned trajectories and partial initialization on ill-conditioned ones.

\subsection{Real-world UAV Experiments}
\begin{figure}[t]
    \centering
    \includegraphics[width=1.0\linewidth]{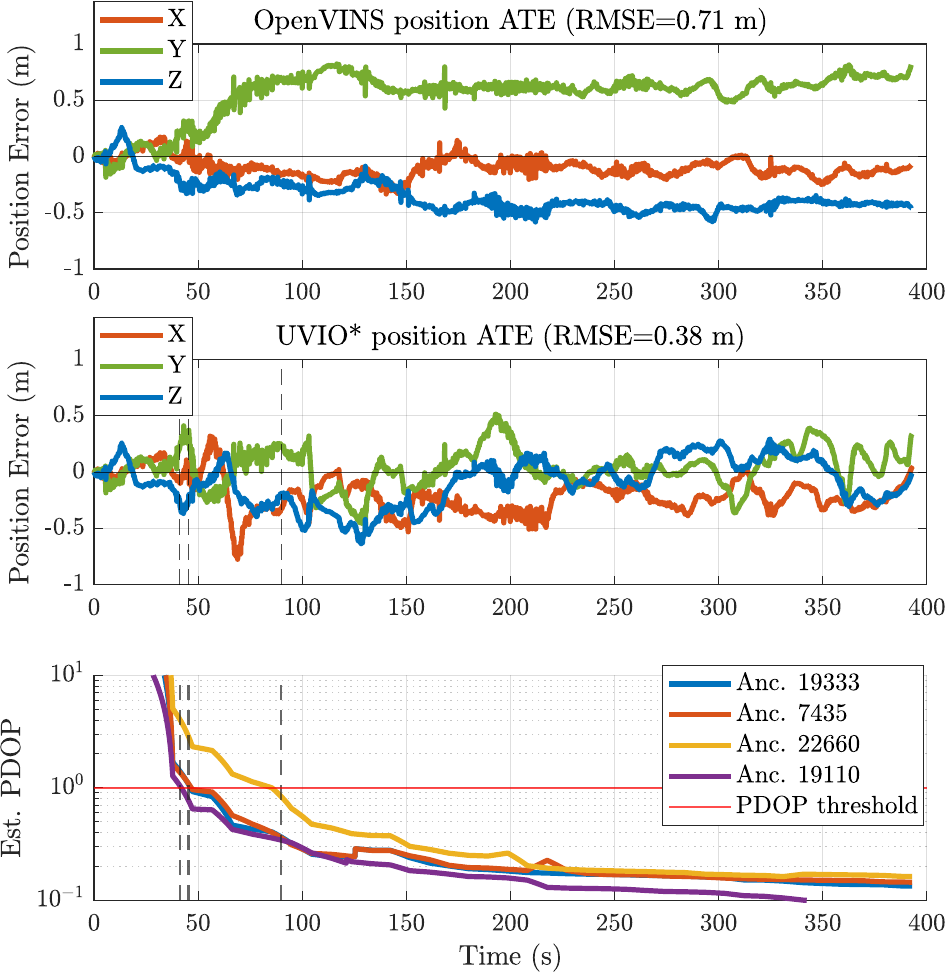}
    \vspace{-7mm}
    \caption{Comparison of OpenVINS and UVIO* performance in a real-world experiment. The top plot displays the position ATE for VIO-only. The middle plot shows the position ATE for UWB-aided VIO using the same experimental data. The bottom plot depicts the evolution of the estimated PDOP for each anchor, with initialization events indicated by dashed lines triggered when the PDOP drops below a threshold of 1. The $46\%$ reduction in position RMSE highlights the importance of fast and accurate initialization.}
    \label{fig:ATE}
    \vspace{-4mm}
\end{figure}
To test our framework with a UAV, we replicated the experiments conducted in~\cite{Delama2023UVIO:Initialization}.
These experiments utilized a UAV equipped with a Raspberry Pi 4 flight computer, an IMU, a Matrix Vision BlueFOX camera, and a Qorvo MDEK1001 UWB transceiver, communicating with four additional UWB modules deployed as fixed anchors.
The UAV autonomously navigated through an indoor space measuring $4\times6.5\times7~m^3$, where an Optitrack motion capture system was used to record the ground-truth poses of both UAV and anchors precisely.
The experiments are structured in two stages: an autonomous phase where the UAV flies through random waypoints and a manual phase where a pilot performs more agile maneuvers.
\begin{table}[t]
    \centering
    \caption{UAV real-world experiments: evaluating in-flight anchor initialization errors and times.}
    \begin{tabular}{|>{\centering\arraybackslash}p{0.2cm}|>{\centering\arraybackslash}p{1.3cm}|>{\centering\arraybackslash}p{0.8cm}|>{\centering\arraybackslash}p{0.8cm}|>{\centering\arraybackslash}p{0.8cm}|>{\centering\arraybackslash}p{0.8cm}|>{\centering\arraybackslash}p{0.8cm}|>{\centering\arraybackslash}p{0.8cm}|}
    \hline
        \# & Anchor ID &  19333 & 22660 & 7435 & 19110 & Avg. \\ \hline\hline
        \multirow {2}{*}{1} & \textbf{Err. ($m$)} & \textbf{0.119} & \textbf{0.155} & \textbf{0.139} & \textbf{0.398} & \textbf{0.203} \\ 
        & $t_{init}$ ($s$) & 40.77 & 137.52 & 95.05 & 40.77 & 78.53 \\ \hline
        \multirow {2}{*}{2} & \textbf{Err. ($m$)} & \textbf{0.126} & \textbf{0.169} & \textbf{0.077} & \textbf{0.196} & \textbf{0.142} \\ 
        & $t_{init}$ ($s$) & 62.94 & 106.85 & 167.79 & 58.55 & 99.03 \\ \hline
        \multirow {2}{*}{3} & \textbf{Err. ($m$)} & \textbf{0.089} & \textbf{0.110} & \textbf{0.117} & \textbf{0.136} & \textbf{0.113} \\ 
        & $t_{init}$ ($s$) & 45.44 & 89.88 & 45.44 & 41.21 & 55.49 \\ \hline
        \multirow {2}{*}{4} & \textbf{Err. ($m$)} & \textbf{0.112} & \textbf{0.169} & \textbf{0.159} & \textbf{0.396} & \textbf{0.209} \\ 
        & $t_{init}$ ($s$) & 73.22 & 166.01 & 166.01 & 73.22 & 119.61 \\ \hline
    \end{tabular}
    \label{tab:rwerrors}
    \vspace{-3mm}
\end{table}

To evaluate initialization performance with real-world UWB range data, we conduct a preliminary set of experiments using the UAV's ground-truth pose.
This approach eliminates the impact of VIO drift, which affects the validation of the initialization phase.
We perform four experiments with two distinct anchor configurations.
All four anchors were initialized in all experiments with an estimated PDOP threshold set to $1$.
The results in \tabref{rwerrors} confirm that the average initialization errors are consistent with those observed in simulations.
Our method demonstrates superior performance compared to the results presented in~\cite{Delama2023UVIO:Initialization}, with shorter initialization time ($88~s$ vs. $220~s$ avg.) and reduced error ($0.167~m$ vs. $0.251~m$ avg.).
Fast and accurate anchor initialization reduces drift in UWB-aided VIO systems, as presented in the following.
\begin{figure}[t]
    \centering
    \includegraphics[width=1.0\linewidth]{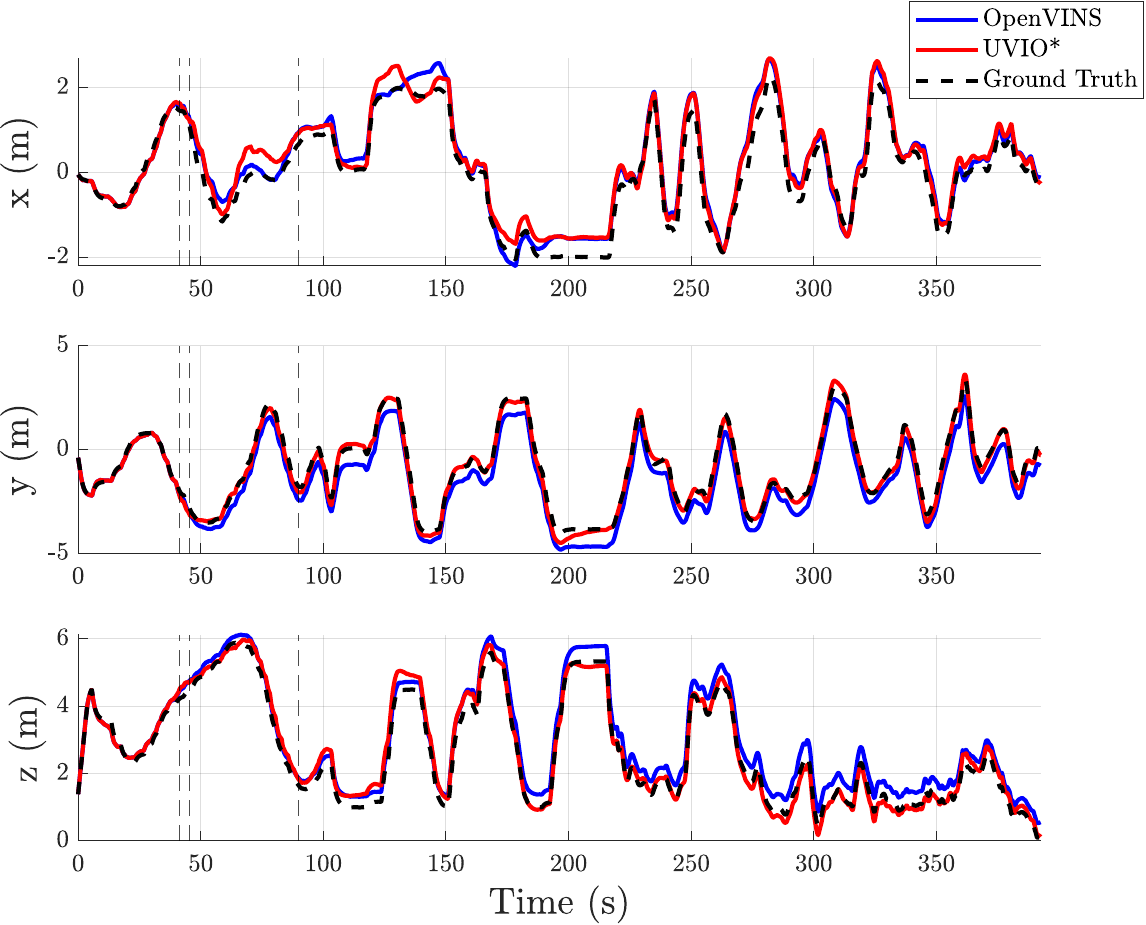}
    \vspace{-8mm}
    \caption{Comparison of ground-truth and estimated trajectories for the real-world experiment 2 described in the last part of the section and for which the ATE is shown in \figref{ATE}. UWB-aided VIO shows enhanced accuracy. Vertical dashed lines indicate the timestamps of UWB anchor initialization.}
    \label{fig:position}
    \vspace{-3mm}
\end{figure}
\begin{table}[t]
    \centering
    \caption{UAV real-world experiments: performance comparison between OpenVINS~\cite{GenevaOpenVINS:Estimation} and UVIO~\cite{Delama2023UVIO:Initialization} with our method.}
    \begin{tabular}{|>{\centering\arraybackslash}p{0.2cm}|>{\centering\arraybackslash}p{1.6cm}|>{\centering\arraybackslash}p{2.0cm}|>{\centering\arraybackslash}p{2.0cm}|}
    \hline
          \# & System & $\mathrm{RMSE}_{pos}$ ($m$) & $\mathrm{RMSE}_{att}$ ($deg$) \\ \hline\hline
          \multirow {2}{*}{1}
          & OpenVINS & 0.73 & 4.47 \\
          & UVIO* & \textbf{0.50} & \textbf{4.14} \\ \hline
          \multirow {2}{*}{2}
          & OpenVINS & 0.71 & 4.67 \\
          & UVIO* & \textbf{0.38} & \textbf{4.56} \\ \hline
          \multirow {2}{*}{3}
          & OpenVINS & 0.72 & \textbf{4.35} \\
          & UVIO* & \textbf{0.34} & 4.5 \\ \hline
          \multirow {2}{*}{4}
          & OpenVINS & 0.71 & 4.44 \\
          & UVIO* & \textbf{0.57} & \textbf{3.64} \\ \hline
    \end{tabular}
    \label{tab:rwexp}
    \vspace{-3mm}
\end{table}

We evaluate real-time performance within the UWB-aided VIO system UVIO~\cite{Delama2023UVIO:Initialization}, modifying it into UVIO* by integrating the proposed real-time anchor calibration framework.
Based on OpenVINS~\cite{GenevaOpenVINS:Estimation}, UVIO* initializes anchors by using VIO-estimated poses.
Performance is assessed via Average Trajectory Error (ATE) in two setups: standard VIO (OpenVINS) and UVIO*, both using identical VIO parameters.
In fact, without initialized UWB anchors, UVIO* defaults to OpenVINS, reflecting its extension of this framework.
\autoref{fig:ATE} shows a significant $46\%$ reduction in the Root Mean Squared Error (RMSE) of the position ATE compared to VIO-only.
\autoref{fig:position} displays the x-y-z trajectory plots corresponding to the experiment in~\figref{ATE}.
This improvement surpasses the results in~\cite{Delama2023UVIO:Initialization}, highlighting the importance of timely initialization for drift-prone localization methods like VIO.
Unlike~\cite{Delama2023UVIO:Initialization}, our method rapidly and precisely estimates unknown UWB anchors within seconds of flight, enabling the system to leverage initialized anchors sooner and effectively mitigating drift accumulation.
The key to this efficiency is our method’s online PDOP estimation of \secref{pdop}, which replaces the extended initialization process used in~\cite{Delama2023UVIO:Initialization}.
\autoref{tab:rwexp} summarizes position and attitude RMSE across four experiments, demonstrating the superior performance of UVIO*.

\section{Conclusion}
\label{sec:conclusion}
This paper proposes an automated framework for the calibration of unknown anchors in UWB-aided navigation systems.
A key contribution is a novel real-time PDOP estimation proved to be underconfident in practical scenarios, ensuring a conservative and reliable initialization decision without needing an initial anchor position guess.
By only triggering initialization when a well-conditioned geometric configuration is achieved, our approach prevents poor anchor estimates that could degrade navigation accuracy.
The framework integrates an online lightweight outlier rejection method and an adaptive robust kernel for nonlinear optimization, further improving robustness.
Extensive real-world experiments with an autonomous forklift and a UAV show that our PDOP-based strategy leads to more accurate positioning and significantly faster initialization compared to state-of-the-art methods.
The open-source C++ library with a ROS1 wrapper makes the implementation widely accessible to the community.

\section{Acknowledgement}
The authors would like to thank the members of the Center for Vision, Automation \& Control of the Austrian Institute of Technology (AIT) under the lead of Csaba Beleznai for assistance in recording the forklift outdoor dataset.

\bibliographystyle{IEEEtran}
\bibliography{bibliography/uwb.bib, Overleaf/bibliography/extra.bib}

\begin{thebibliography}{10}
\providecommand{\url}[1]{#1}
\csname url@samestyle\endcsname
\providecommand{\newblock}{\relax}
\providecommand{\bibinfo}[2]{#2}
\providecommand{\BIBentrySTDinterwordspacing}{\spaceskip=0pt\relax}
\providecommand{\BIBentryALTinterwordstretchfactor}{4}
\providecommand{\BIBentryALTinterwordspacing}{\spaceskip=\fontdimen2\font plus
\BIBentryALTinterwordstretchfactor\fontdimen3\font minus \fontdimen4\font\relax}
\providecommand{\BIBforeignlanguage}[2]{{%
\expandafter\ifx\csname l@#1\endcsname\relax
\typeout{** WARNING: IEEEtran.bst: No hyphenation pattern has been}%
\typeout{** loaded for the language `#1'. Using the pattern for}%
\typeout{** the default language instead.}%
\else
\language=\csname l@#1\endcsname
\fi
#2}}
\providecommand{\BIBdecl}{\relax}
\BIBdecl

\bibitem{Chebrolu2021AdaptiveProblems}
N.~Chebrolu, T.~Labe, O.~Vysotska, J.~Behley, and C.~Stachniss, ``{Adaptive Robust Kernels for Non-Linear Least Squares Problems},'' \emph{IEEE Robotics and Automation Letters}, vol.~6, no.~2, pp. 2240--2247, 4 2021.

\bibitem{Delama2023UVIO:Initialization}
G.~Delama, F.~Shamsfakhr, S.~Weiss, D.~Fontanelli, and A.~Fomasier, ``{UVIO: An UWB-Aided Visual-Inertial Odometry Framework with Bias-Compensated Anchors Initialization},'' \emph{IEEE International Conference on Intelligent Robots and Systems}, pp. 7111--7118, 2023.

\bibitem{Pelka2016IterativeSystems}
M.~Pelka, G.~Goronzy, and H.~Hellbr{\"{u}}ck, ``{Iterative approach for anchor configuration of positioning systems},'' 2016.

\bibitem{Krapez2020AnchorSystems}
P.~Krape{\v{z}} and M.~Munih, ``{Anchor Calibration for Real-Time-Measurement Localization Systems},'' \emph{IEEE Transactions on Instrumentation and Measurement}, vol.~69, no.~12, 2020.

\bibitem{Nguyen2021VIRALSLAM}
T.~M. Nguyen, S.~Yuan, M.~Cao, T.~H. Nguyen, and L.~Xie, ``{VIRAL SLAM: Tightly Coupled Camera-IMU-UWB-Lidar SLAM},'' 5 2021.

\bibitem{Nguyen2022VIRAL-Fusion:Approach}
T.~M. Nguyen, M.~Cao, S.~Yuan, Y.~Lyu, T.~H. Nguyen, and L.~Xie, ``{VIRAL-Fusion: A Visual-Inertial-Ranging-Lidar Sensor Fusion Approach},'' \emph{IEEE Transactions on Robotics}, vol.~38, no.~2, pp. 958--977, 4 2022.

\bibitem{Herbruggen2023MultihopPositioning}
B.~V. Herbruggen, S.~Luchie, J.~Fontaine, and E.~De~Poorter, ``{Multihop Self-Calibration Algorithm for Ultra-Wideband (UWB) Anchor Node Positioning},'' \emph{IEEE Journal of Indoor and Seamless Positioning and Navigation}, vol.~1, pp. 1--11, 5 2023.

\bibitem{Ridolfi2021UWBApproach}
M.~Ridolfi, J.~Fontaine, B.~V. Herbruggen, W.~Joseph, J.~Hoebeke, and E.~D. Poorter, ``{UWB anchor nodes self-calibration in NLOS conditions: a machine learning and adaptive PHY error correction approach},'' \emph{Wireless Networks}, vol.~27, no.~4, pp. 3007--3023, 5 2021.

\bibitem{Corbalan2023Self-LocalizationPractice}
P.~Corbal{\'{a}}n, G.~P. Picco, M.~Coors, and V.~Jain, ``{Self-Localization of Ultra-Wideband Anchors: From Theory to Practice},'' \emph{IEEE Access}, vol.~11, pp. 29\,711--29\,725, 2023.

\bibitem{Qi2024CalibrationLocalization}
M.~Qi, B.~Xue, W.~Wang, and S.~Member, ``{Calibration and Compensation of Anchor Positions for UWB Indoor Localization},'' \emph{IEEE SENSORS JOURNAL}, vol.~24, no.~1, 2024.

\bibitem{Mahmoud2022Ultra-widebandTracking}
A.~Mahmoud, P.~Coser, H.~Sadruddin, and M.~Atia, ``{Ultra-wideband Automatic Anchor's Localization for Indoor Path Tracking},'' \emph{Proceedings of IEEE Sensors}, vol. 2022-October, 2022.

\bibitem{Hausman2016Self-calibratingUAV}
K.~Hausman, S.~Weiss, R.~Brockers, L.~Matthies, and G.~S. Sukhatme, ``{Self-calibrating multi-sensor fusion with probabilistic measurement validation for seamless sensor switching on a UAV},'' \emph{Proceedings - IEEE International Conference on Robotics and Automation}, vol. 2016-June, pp. 4289--4296, 6 2016.

\bibitem{Batstone2017TowardsAnchors}
K.~Batstone, M.~Oskarsson, and K.~{\AA}str{\"{o}}m, ``{Towards real-time time-of-arrival self-calibration using ultra-wideband anchors},'' \emph{2017 International Conference on Indoor Positioning and Indoor Navigation, IPIN 2017}, vol. 2017-January, pp. 1--8, 11 2017.

\bibitem{Shi2019AnchorMeasurements}
Q.~Shi, S.~Zhao, X.~Cui, M.~Lu, and M.~Jia, ``{Anchor self-localization algorithm based on UWB ranging and inertial measurements},'' \emph{Tsinghua Science and Technology}, vol.~24, no.~6, pp. 728--737, 12 2019.

\bibitem{Nguyen2020Tightly-CoupledSystem}
T.~H. Nguyen, T.~M. Nguyen, and L.~Xie, ``{Tightly-Coupled Single-Anchor Ultra-wideband-Aided Monocular Visual Odometry System},'' \emph{Proceedings - IEEE International Conference on Robotics and Automation}, pp. 665--671, 5 2020.

\bibitem{Nguyen2021Range-FocusedLocalization}
------, ``{Range-Focused Fusion of Camera-IMU-UWB for Accurate and Drift-Reduced Localization},'' \emph{IEEE Robotics and Automation Letters}, vol.~6, no.~2, pp. 1678--1685, 4 2021.

\bibitem{Gao2022LowLocalization}
B.~Gao, B.~Lian, D.~Wang, and C.~Tang, ``{Low drift visual inertial odometry with UWB aided for indoor localization},'' \emph{IET Communications}, vol.~16, no.~10, pp. 1083--1093, 6 2022.

\bibitem{Jia2022FEJ-VIRO:Odometry}
S.~Jia, Y.~Jiao, Z.~Zhang, R.~Xiong, and Y.~Wang, ``{FEJ-VIRO: A Consistent First-Estimate Jacobian Visual-Inertial-Ranging Odometry},'' \emph{IEEE International Conference on Intelligent Robots and Systems}, vol. 2022-October, pp. 1336--1343, 2022.

\bibitem{Jia2023DistributedNetwork}
S.~Jia, R.~Xiong, and Y.~Wang, ``{Distributed Initialization for Visual-Inertial-Ranging Odometry with Position-Unknown UWB Network},'' \emph{Proceedings - IEEE International Conference on Robotics and Automation}, vol. 2023-May, pp. 6246--6252, 2023.

\bibitem{Li2023UWB-VO:Odometry}
K.~Li, S.~Bu, Y.~Dong, Y.~Wang, X.~Jia, and Z.~Xia, ``{UWB-VO: Ultra-Wideband Anchor Assisted Visual Odometry},'' \emph{Proceedings of 2023 IEEE International Conference on Unmanned Systems, ICUS 2023}, pp. 943--950, 2023.

\bibitem{Hamesse2024FastSystem}
C.~Hamesse, R.~Vleugels, M.~Vlaminck, H.~Luong, and R.~Haelterman, ``{Fast and Cost-Effective UWB Anchor Position Calibration Using a Portable SLAM System},'' \emph{IEEE Sensors Journal}, 2024.

\bibitem{Hu2023TightlyAnchors}
C.~Hu, P.~Huang, W.~Wang, and S.~Member, ``{Tightly Coupled Visual-Inertial-UWB Indoor Localization System With Multiple Position-Unknown Anchors},'' \emph{IEEE Robotics and Automation Letters}, vol.~9, no.~1, 2023.

\bibitem{Blueml2021BiasPoints}
J.~Blueml, A.~Fornasier, and S.~Weiss, ``{Bias Compensated UWB Anchor Initialization using Information-Theoretic Supported Triangulation Points},'' \emph{Proceedings - IEEE International Conference on Robotics and Automation}, vol. 2021-May, pp. 5490--5496, 2021.

\bibitem{Hu2023RobustUAVs}
J.~Hu, Y.~Li, Y.~Lei, Z.~Xu, M.~Lv, and J.~Han, ``{Robust and Adaptive Calibration of UWB-Aided Vision Navigation System for UAVs},'' \emph{IEEE Robotics and Automation Letters}, vol.~8, no.~12, 2023.

\bibitem{Jung2024ModularCalibration}
R.~Jung, L.~Santoro, D.~Brunelli, D.~Fontanelli, and S.~Weiss, ``{Modular Meshed Ultra-Wideband Aided Inertial Navigation with Robust Anchor Calibration},'' 8 2024.

\bibitem{Luo2025Visual-inertialPosition}
H.~Luo, D.~Zou, J.~Li, A.~Wang, L.~Wang, Z.~Yang, and G.~Li, ``{Visual-inertial navigation assisted by a single UWB anchor with an unknown position},'' \emph{Satellite Navigation}, vol.~6, no.~1, pp. 1--21, 12 2025.

\bibitem{Sun2024AEnvironments}
J.~Sun, W.~Sun, J.~Zheng, Z.~Chen, C.~Tang, and X.~Zhang, ``{A Novel UWB/IMU/Odometer-Based Robot Localization System in LOS/NLOS Mixed Environments},'' 2024.

\bibitem{Fontanelli2021Cramer-RaoG-WLS}
D.~Fontanelli, F.~Shamsfakhr, and L.~Palopoli, ``{Cramer-Rao Lower Bound Attainment in Range-Only Positioning Using Geometry: The G-WLS},'' \emph{IEEE Transactions on Instrumentation and Measurement}, vol.~70, 2021.

\bibitem{GenevaOpenVINS:Estimation}
P.~Geneva, K.~Eckenhoff, W.~Lee, Y.~Yang, and G.~Huang, \emph{{OpenVINS: A Research Platform for Visual-Inertial Estimation}}.

\end{thebibliography}
%%%%%%%%%%%%%%%%%%%%%%%%%%%%%%%%%%%%%%%%%%%%%%%%%%%%%%%%%%%%%%%%%%%%%%%%%%%%%%%%

\end{document}